\documentclass[letterpaper]{article} 
\usepackage{aaai23}  
\usepackage{times}  
\usepackage{helvet}  
\usepackage{courier}  
\usepackage[hyphens]{url}  
\usepackage{graphicx} 
\usepackage{natbib}  
\usepackage{caption} 
\usepackage{algorithm}
\usepackage{algorithmic}

\usepackage{newfloat}
\usepackage{listings}
\DeclareCaptionStyle{ruled}{labelfont=normalfont,labelsep=colon,strut=off} 
\floatstyle{ruled}
\newfloat{listing}{tb}{lst}{}
\floatname{listing}{Listing}
\usepackage{amsfonts} 
\usepackage{amsmath}
\usepackage{amsthm} 
\usepackage{algorithm}
\usepackage{algorithmic}
\usepackage{subcaption} 
\usepackage{tabularx}
\usepackage{subcaption}

\usepackage{booktabs}
\usepackage{diagbox}
\newtheorem{theorem}{Theorem}
\newtheorem{lemma}{Lemma}
\newtheorem{assumption}{Assumption}

\title{Learning Explicit Credit Assignment for Cooperative Multi-Agent Reinforcement Learning via Polarization Policy Gradient}

\author{
    Wubing Chen\textsuperscript{\rm 1},
    Wenbin Li\textsuperscript{\rm 1} \thanks{Corresponding authors.},
    Xiao Liu\textsuperscript{\rm 1},
    Shangdong Yang\textsuperscript{\rm 2, \rm 1},
    Yang Gao\textsuperscript{\rm 1 $*$}
}
\affiliations{
    \textsuperscript{\rm 1}State Key Laboratory for Novel Software Technology, Nanjing University, Nanjing 210023, China\\
    \textsuperscript{\rm 2}Nanjing University of Posts and Telecommunications, Nanjing 210023, China\\
    wuzbingchen@gmail.com,  liwenbin@nju.edu.cn,  liuxiao730@outlook.com, sdyang@njupt.edu.cn, gaoy@nju.edu.cn
}

\usepackage{bibentry}

\begin{document}
\maketitle

\begin{abstract}
Cooperative multi-agent policy gradient (MAPG) algorithms have recently attracted wide attention and are regarded as a general scheme for the multi-agent system. Credit assignment plays an important role in MAPG and can induce cooperation among multiple agents. However, most MAPG algorithms cannot achieve good credit assignment because of the game-theoretic pathology known as \textit{centralized-decentralized mismatch}. To address this issue, this paper presents a novel method,  \textit{\underline{M}ulti-\underline{A}gent \underline{P}olarization \underline{P}olicy \underline{G}radient} (MAPPG). MAPPG  takes a simple but efficient polarization function to transform the optimal consistency of joint and individual actions into easily realized constraints, thus enabling efficient credit assignment in MAPG. Theoretically, we prove that  individual policies of MAPPG can converge to the global optimum. Empirically,  we evaluate MAPPG on the  well-known matrix game and differential game, and verify that MAPPG can converge to the global optimum for both discrete and continuous action spaces.  We also evaluate MAPPG on a set of StarCraft II micromanagement tasks and demonstrate that MAPPG outperforms  the state-of-the-art MAPG algorithms.
\end{abstract}

\section{Introduction}
\label{sec:introdiction}
Multi-agent reinforcement learning (MARL) is a critical learning technology to solve sequential decision problems with multiple agents.  
Recent developments in MARL have heightened the need for fully cooperative MARL that maximizes a reward shared by all agents. Cooperative MARL has made remarkable advances in many domains, including autonomous driving \cite{DBLP:conf/icra/CaoWDY21} and  cooperative transport \cite{DBLP:conf/icra/ShibataJM21}. To mitigate the combinatorial nature \cite{DBLP:journals/aamas/Hernandez-LealK19} and partial observability \cite{DBLP:conf/icml/OmidshafieiPAHV17} in MARL, \textit{centralized training with decentralized execution} (CTDE) \cite{DBLP:journals/jair/OliehoekSV08,DBLP:journals/ijon/KraemerB16} has become one of the mainstream settings for MARL, where  global information is provided to promote collaboration in the training phase and learned policies are executed only based on local observations. 

Multi-agent credit assignment is a crucial challenge in the MARL under the CTDE setting, which refers to attributing a global environmental reward to the individual agents' actions \cite{lica}. Multiple independent agents can learn effective collaboration policies to accomplish challenging tasks with the proper credit assignment.  
MARL algorithms can be divided into value-based and policy-based. 
Cooperative multi-agent policy gradient (MAPG) algorithms can handle both discrete  and continuous action spaces, which  is the focus of our study. Different MAPG algorithms adopt different credit assignment paradigms, which can be divided into implicit and explicit credit assignment \cite{lica}. 
Solving the  credit assignment problem implicitly needs to represent the joint action value as a function of  individual policies \cite{maddpg,lica,dop,fop,facmac}. Current state-of-the-art MAPG algorithms \cite{dop,fop,facmac} impose a monotonic constraint between the joint action value and individual policies.  While some algorithms allow more expressive value function classes, the capacity of the value mixing network is still limited by the monotonic constraints \cite{qtran,qplex}. 
The other algorithms that achieve explicit credit assignment mainly provide a shaped reward for each individual agent's action  \cite{DBLP:conf/aamas/ProperT12,coma,vdac}. However,  there is a large discrepancy between the performance of  algorithms with explicit credit assignment  and   algorithms with implicit credit assignment. 

In this paper, we analyze this discrepancy and pinpoint that the \textit{centralized-decentralized mismatch} hinders the performance of MAPG algorithms with explicit credit assignment. The \textit{centralized-decentralized mismatch} can arise when the sub-optimal  policies of agents could negatively affect the assessment of other agents' actions, which leads to catastrophic miscoordination. Note that the issue of \textit{centralized-decentralized mismatch} was raised by DOP \cite{dop}. However, the linearly decomposed critic adopted by DOP \cite{dop}  limits their representation expressiveness for  the value function.   

Inspired by Polarized-VAE \cite{polar-vae} and Weighted QMIX \cite{weighted-qmix}, we propose a policy-based algorithm called \textit{\underline{M}ulti-\underline{A}gent \underline{P}olarization \underline{P}olicy \underline{G}radient} (MAPPG) for learning explicit credit assignment to address the \textit{centralized-decentralized mismatch}.  MAPPG  encourages increasing the distance between the global optimal joint action value and the non-optimal joint action values while shortening the distance between multiple non-optimal joint action values via polarization policy gradient. 
MAPPG facilitates large-scale multi-agent cooperations and presents a new multi-agent credit assignment paradigm, enabling multi-agent policy learning like single-agent policy learning \cite{DBLP:journals/jmlr/WeiL16}. 
Theoretically, we prove that  individual policies of MAPPG can converge to the global optimum. 
Empirically,  we verify that MAPPG can converge to the global optimum compared to existing MAPG algorithms in the well-known matrix \cite{qtran} and differential games \cite{DBLP:conf/aaaiss/WeiWFL18}. We also show that MAPPG outperforms  the state-of-the-art MAPG algorithms on StarCraft II unit micromanagement tasks \cite{smac}, demonstrating its scalability in complex scenarios. Finally, the results of ablation experiments match our theoretical predictions.
\section{Related Work}
\label{sec:related_work}
\subsection{Implicit Credit Assignment}
In general, implicit MAPG algorithms utilize the learned function between the individual policies and the joint action values for credit assignment. MADDPG \cite{maddpg} and LICA \cite{lica}  learn the individual policies by directly ascending the approximate joint action value gradients. The state-of-the-art MAPG algorithms \cite{dop,fop,facmac,vdac}  introduce the idea of value function decomposition \cite{vdn,qmix,qtran,qplex,weighted-qmix} into the multi-agent actor-critic framework. DOP \cite{dop} decomposes the centralized critic as a weighted linear summation of individual critics that condition local actions. FOP \cite{fop} imposes a multiplicative form between the optimal joint policy and the individual optimal policy, and optimizes both  policies based  on maximum entropy reinforcement learning objectives. FACMAC \cite{facmac} proposes a new credit-assignment actor-critic framework that  factors the joint  action value  into individual action values and uses the  centralized gradient estimator for credit assignment. VDAC \cite{vdac} achieves the credit assignment by  enforcing the monotonic relationship between the joint action values and the shaped individual action values. Although these algorithms  allow more expressive value function classes, the capacity of the value mixing network is still limited by the monotonic constraints, and this claim  will be verified in our experiments.
\subsection{Explicit Credit Assignment}
In contrast to implicit  algorithms, explicit MAPG algorithms provide the contribution of each individual agent's action, and the individual actor is updated by following policy gradients tailored by the contribution. COMA \cite{coma} evaluates the contribution of individual agents' actions by using the centralized critic to compute an agent-specific advantage function. SQDDPG \cite{sqddpg} proposes a local reward algorithm, Shapley Q-value, which takes the expectation of marginal contributions of all possible coalitions.  Although explicit algorithms provide valuable insights into the assessment of the contribution of individual agents' actions to the global reward and thus can significantly facilitate policy optimization,  the issue of
\textit{centralized-decentralized mismatch}  hinders their performance in complex scenarios.  Compared to explicit algorithms, the proposed MAPPG can theoretically tackle the challenge of \textit{centralized-decentralized mismatch}  and experimentally outperforms existing MAPG algorithms in both convergence speed and final performance  in challenging environments.
\section{Background}
\label{sec:background}
\subsection{Dec-POMDP}
A decentralized  partially observable Markov decision process (Dec-POMDP) is a tuple $\left< S,{U},r,P,Z,O,n,\gamma \right> $, where  $n$ agents  identified by $a\in A\equiv \{1,...,n\}$ choose sequential actions, $s\in S$ is the state. At each time step, each agent chooses an action $u_a \in U$, forming a joint action $\boldsymbol{u}\in \boldsymbol{U}\equiv U^n$ which induces a transition in the environment according to the state transition function 
$P\left(s^{\prime} \mid s, \boldsymbol{u}\right): S \times \boldsymbol{U} \times S \rightarrow[0,1]$. Agents  receive the same  reward  according to the reward function $r(s, \boldsymbol{u}): S \times \boldsymbol{U}\rightarrow \mathbb{R}$. Each agent has an observation function $O(s, a): S \times A \rightarrow Z$, where a partial observation $z_a\in Z$   is drawn. $\gamma \in [0,1) $ is the discount factor.  Throughout this paper, we denote joint quantities over agents in bold, quantities with the subscript $a$ denote  quantities over agent $a$, and joint quantities over agents other than a given agent $a$ with the subscript $-a$.
Each agent tries to learn a stochastic policy for action selection: ${\pi}_a: T \times U \rightarrow [0,1]$, where $\tau_{a} \in T \equiv(Z \times U)^{*}$ is an action-observation history for agent $a$.
MARL agents try to maximize   the cumulative return, ${R}^{t}=\sum_{t=1}^{\infty} \gamma^{t-1} r^{t}$, where $r^t$ is the reward obtained from the environment by all agents at step $t$.  
\subsection{Multi-Agent Policy Gradient}
We first provide the background on single-agent policy gradient algorithms, and then introduce multi-agent policy gradient algorithms. In single-agent continuous control tasks, policy gradient algorithms \cite{pg} optimise a single agent's policy, parameterised by $\theta$, by performing gradient ascent on an estimator of the expected discounted total reward $\nabla _{\theta}J\left( \pi \right) =\mathbb{E}_{\pi}\left[ \nabla _{\theta}\log \pi \left( u\mid s \right) R^0 \right] $, where the gradient is estimated from trajectories sampled from the environment. Actor-critic \cite{pg,ac,DBLP:journals/corr/SchulmanMLJA15} algorithms use an estimated action value instead of the discounted return to solve the high variance caused by the likelihood-ratio trick in the above formula. The gradient of the policy  for  a single-agent setting  can be defined as: 
\begin{align}
	\nabla _{\theta}J\left( \pi \right) =\mathbb{E}_{\boldsymbol{\pi }}\left[ \nabla _{\theta}\log \pi \left( u\mid s \right) Q\left( s,u \right) \right].
\end{align}
A natural extension to multi-agent settings leads to the multi-agent stochastic policy gradient theorem  with agent $ a$'s policy parameterized by $\theta_a$  \cite{coma,DBLP:conf/aaaiss/WeiWFL18}, shown below:
\begin{align}
	\nabla _{\theta}J\left( \boldsymbol{\pi } \right)  =&
	\mathbb{E}_{\boldsymbol{\pi }}\left[ \sum_a{\nabla _{\theta_a}\log \pi _a\left( u_a\mid \tau_a \right) Q\left( s,\boldsymbol{u} \right)} \right] \nonumber
	\\
	=& \sum_s{d^{\boldsymbol{\pi }}\left( s \right)}\sum_a{\sum_{{u}_a}{\pi _a\left( u_a|\tau_a \right) \nabla _{\theta_a}\log \pi _a\left( u_a|\tau_a \right)}}  \nonumber \\
	& \sum_{\boldsymbol{u}_{-a}}{\boldsymbol{\pi }_{-a}\left( \boldsymbol{u}_{-a}|\boldsymbol{\tau}_{-a} \right)}Q\left( s,\boldsymbol{u} \right).
	\label{eq:mapg}
\end{align}
where $d^{\boldsymbol{\pi}}\left( s \right)$ is a discounted weighting of states encountered starting at $s^0$ and then following $\boldsymbol{\pi}$. COMA  implements the multi-agent stochastic policy gradient theorem by replacing the action-value function with the counterfactual advantage, reducing variance, and not changing the expected gradient.
\section{Analysis}
\label{sec:analysis}
In the multi-agent stochastic policy gradient theorem,  agent $a$ learns the policy by directly ascending the approximate marginal joint action value gradient for each $u_a\in U$, which is scaled by $M_a(s,u_a,\boldsymbol{\pi}_{-a})=\sum_{\boldsymbol{u}_{-a}}{\boldsymbol{\pi }_{-a}}\left( \boldsymbol{u}_{-a}\mid \boldsymbol{\tau}_{-a} \right) Q\left( s,\boldsymbol{u} \right)  $ (see Equation (\ref{eq:mapg})). Formally, suppose that the optimal and a non-optimal joint action under $s$ are $\boldsymbol{u}^*$ and $\boldsymbol{u}^{\#}$ respectively, that is, $Q(s,\boldsymbol{u}^*)>Q(s,\boldsymbol{u}^{\#})$. If it holds that $M_a(s,{u}_a^*,\boldsymbol{\pi }_{-a})<M_a(s,{u}_a^{\#},\boldsymbol{\pi }_{-a})$ due to the exploration or suboptimality of other agents' policies, we possibly have that $\pi_a(u_a^{*} \mid \tau_a)<\pi_a(u_a^{\#} \mid \tau_a)$. The decentralized policy of agent $a$ is updated by following policy gradients tailored by the centralized critic, which are negatively affected by other agents' policies. This issue is called \textit{centralized-decentralized mismatch} \cite{dop}.  We will show that \textit{centralized-decentralized mismatch} occurs in practice for the state-of-the-art MAPG algorithms on the well-known matrix game and differential game in the experimental section.

\section{The Proposed Method}
\label{sec:mappg}
In this section, we first propose a novel multi-agent actor-critic method, MAPPG, which learns  explicit credit assignment. Then we mathematically prove that MAPPG can address the issue of \textit{centralized-decentralized mismatch} and the individual policies of MAPPG can  converge to the global optimum.
\subsection{The Polarization  Policy Gradient}
\label{sec:ppg}
In the multi-agent stochastic policy gradient theorem, the scale of the policy gradient of agent $a$ is impacted by the policies of other agents, which leads to  \textit{centralized-decentralized mismatch}.  A straightforward solution is to make the policies of other agents optimal. Learning other agents' policies depends on the convergence of agent $ a$'s policy to the optimal.  If agent $ a$'s policy converges to the optimal, there seems no need to compute the scale of the policy gradient of agent $a$. Therefore, we cannot solve the problem of  \textit{centralized-decentralized mismatch} from a policy perspective, we seek to address it from the joint action value perspective. 

We define polarization joint action values to replace  original joint action values. The polarization joint action values  resolve \textit{centralized-decentralized mismatch} by increasing the distance between the values of the global optimal joint action and the non-optimal joint actions while shortening the distance between the values of multiple non-optimal joint actions.  By polarization, the influence of other agents' non-optimal policies can be largely eliminated. For convenience, the following discussion in the section will assume that  the action values are fixed in a given state $s$. In later sections, we will see how the joint action values are updated. If the optimal joint action $\boldsymbol{u}^{*}$ in state $s$ can be identified, then the polarization policy gradient is:
\begin{equation}
\begin{aligned}
	\nabla _{\theta}J\left( \boldsymbol{\pi } \right) = \mathbb{E}_{\boldsymbol{\pi }}\left[ \sum_a{\nabla _{\theta _a}\log \pi _a\left( u_a\mid \tau _a \right) Q^{PPG}\left( s,\boldsymbol{u} \right)} \right] ,   \nonumber
\end{aligned}            
\end{equation}
where
\begin{equation}
\begin{aligned}
	Q^{PPG}\left( s,\boldsymbol{u} \right) =\left\{ \begin{array}{l}
		1\ \text{if}\ \boldsymbol{u}=\boldsymbol{u}^{*}\\
		0\ \text{otherwise}\\
	\end{array} \right. 
\end{aligned}            
\end{equation}
is the polarization joint action value function. For each agent, only the gradient of the component of the optimal action is greater than 0, whereby the optimal policy can be learned. However, we cannot traverse all state-action pairs to find the optimal joint action $u^*$ in complex scenarios; therefore,  a soft version of the polarization joint action value function is defined as:
\begin{align}
Q^{PPG}\left( s,\boldsymbol{u} \right) =\exp \left( \alpha\ Q\left( s,\boldsymbol{u} \right)   \right), 
\label{eq:exp}
\end{align}
where $\alpha>0$ denotes the enlargement factor determining the distance between the optimal  and the non-optimal joint action values. However, Equation (\ref{eq:exp}) cannot work in practice. On the one hand, the result of the exponential function is easy to overflow. On the other hand, if
$\forall \boldsymbol{u},\ Q\left( s,\boldsymbol{u} \right) \leq 0 $, the polarization joint action values $Q^{PPG}\left( s,\boldsymbol{u} \right)$ are between 0 and 1. To address the polarization failure, a baseline is introduced as follows:
\begin{align}
Q^{PPG}\left( s,\boldsymbol{u} \right) = \frac{1}{\beta}\exp \left( \alpha \left( Q\left( s,\boldsymbol{u} \right) -Q\left( s,\boldsymbol{u}_{curr} \right) \right) \right) ,
\label{eq:exp_with_baseline}
\end{align}
where $\beta$ is a factor which can prevent exponential gradient explosion and $\boldsymbol{u}_{curr}\ =[\arg\max_{u_a}\ \pi_a({u_a|\tau_a})]_{a=1}^n $. By providing a baseline, the policy  is guided to pay more attention to the joint actions of $\{\boldsymbol{u}\ :\ Q\left( s,\boldsymbol{u} \right) >Q\left( s,\boldsymbol{u}_{curr} \right)\} $, which derives a self-improving method. Our method looks similar to COMA , but they are different  in nature.  The baseline in our MAPPG can help solve the \textit{centralized-decentralized mismatch}.  However, the baseline in COMA  is introduced to achieve difference rewards.

Adopting polarization joint action values, MAPPG  solves the credit assignment issue by applying the following polarization policy gradients:
\begin{align}
\nabla _{\theta}J\left( \boldsymbol{\pi } \right)  =&\frac{1}{\beta}
\mathbb{E}_{\boldsymbol{\pi }}\left[ \sum_a{\nabla _{\theta_a}\log \pi _a\left( u_a\mid \tau_a \right) Q^{PPG}\left( s,\boldsymbol{u} \right)} \right] \nonumber \\
=& \frac{1}{\beta} \sum_s{d^{\boldsymbol{\pi }}\left( s \right)}\sum_a{\sum_{{u}_a}{\pi _a\left( u_a|\tau_a \right) \nabla _{\theta_a}\log \pi _a\left( u_a|\tau_a \right)}}\nonumber\\
&  \sum_{\boldsymbol{u}_{-a}}{\boldsymbol{\pi }_{-a}}\left( \boldsymbol{u}_{-a}\mid s \right) Q^{PPG}\left( s,u_a,\boldsymbol{u}_{-a} \right).
\label{eq:pg_ppg}
\end{align}
From Equation (\ref{eq:pg_ppg}), we can see that the  gradient for action $u_a$ at state $s$ is scaled by $M_a^{PPG}(s,u_a,\boldsymbol{\pi}_{-a})= \sum_{\boldsymbol{u}_{-a}}{\boldsymbol{\pi }_{-a}}\left( \boldsymbol{u}_{-a}\mid s \right) Q^{PPG}\left( s,u_a,\boldsymbol{u}_{-a} \right)$, which is the polarization marginal joint action value function.

The power function is adopted because it has two  properties. (i) The second-order gradient of the power function is greater than 0, so it can increase the distance between the global optimal joint action value and the non-optimal joint action values, while shortening the distance between multiple non-optimal joint action values. (ii) For all $\boldsymbol{u}\in \{\boldsymbol{u}:Q\left( s,\boldsymbol{u} \right) <Q\left( s,\boldsymbol{u}_{curr} \right)\} $, the corresponding polarized joint action values $Q^{PPG}\left( s,\boldsymbol{u} \right)$  are between 0 and 1, which makes  the  policy learning  focus more on the domain $\{\boldsymbol{u}:Q\left( s,\boldsymbol{u} \right) > Q\left( s,\boldsymbol{u}_{curr} \right)\} $ in state $s$.
\subsection{Theoretical Proof}
\label{sec:theoretical_proof}
In this subsection, we introduce the joint policy improvement for MAPPG and mathematically prove that  the individual policies of MAPPG can  converge to the global optimum. 	For convenience, this section will be discussed in a fully observable environment, where each agent chooses actions based on the state instead of the action-observation history. To ensure the uniqueness of the optimal joint action, we make the following assumptions. 
\begin{assumption} The joint action value function $Q(s,\boldsymbol{u})$  has one  unique maximizing joint action  for all $s\in S$ and $|\boldsymbol{U}|<\infty$.
\end{assumption}
First, we mathematically prove that each maximizing individual action of the polarization marginal joint action value function is consistent with the maximizing joint action's corresponding  component of the  joint action value function in Theorem \ref{thm:optimality_consistency}. 
\begin{theorem}[Optimality Consistency]
Let $\boldsymbol{\pi}$ be a joint policy. Let $\boldsymbol{u}^{*}=\arg\max_{\boldsymbol{u} \in \boldsymbol{U} }Q(s,\boldsymbol{u})$ and $\boldsymbol{u}^{sec}=\arg\max_{\boldsymbol{u} \in (\boldsymbol{U}-\{\boldsymbol{u}^{*}\})  }Q(s,\boldsymbol{u})$. If  it holds that $\forall a\in A$, $\alpha >\frac{\log \boldsymbol{\pi }_{-a}\left( \boldsymbol{u}_{-a}^{*}|s \right)}{Q\left( s,\boldsymbol{u}^{\sec} \right) -Q\left( s,\boldsymbol{u}^* \right)}$ with $\alpha$ as defined in Equation (\ref{eq:exp_with_baseline}), then we have that for all individual actions $u_a^{\prime}$:
\begin{align}
	M_{a}^{PPG}\left( s,u_a^{\prime},\boldsymbol{\pi}_{-a} \right) < M_{a}^{PPG}\left( s,u_{a}^{*},\boldsymbol{\pi}_{-a} \right) ,\nonumber
\end{align}
\label{thm:optimality_consistency}
where $u_a^{\prime}\ne u_a^{*}$.
\end{theorem}
\begin{proof}
See Appendix A.
\end{proof}
Theorem 1 reveals an important insight. The  enlargement  factor regulates  the distance between the optimal and non-optimal action values, and  MAPPG tackles the challenge of \textit{centralized-decentralized mismatch} with $\alpha >\frac{\log \boldsymbol{\pi }_{-a}\left( \boldsymbol{u}_{-a}^{*}|s \right)}{Q\left( s,\boldsymbol{u}^{\sec} \right) -Q\left( s,\boldsymbol{u}^* \right)}$.

Second, first-order optimization algorithms for training deep neural networks are difficult to ensure global convergence ~\cite{DBLP:books/daglib/0040158}. Hence, one assumption on policy parameterizations  is required for our analysis.
\begin{assumption} Given the function $\psi: S \times U \rightarrow \mathbb{R}$, agent $a$'s policy $\pi_a(\cdot|s)$ is the corresponding vector of action probabilities given by the softmax parameterization for all $u_a^{\prime}\in U$, i.e.,
\begin{align}
	\pi_{a}(u_a|s)=\frac{\exp \left(\psi_{s, u_a}\right)}{\sum_{u_a^{\prime} \in {U}} \exp \left(\psi_{s, u_a^{\prime}}\right)},\nonumber
\end{align}
where $\psi_{s,u_a}\equiv\psi(s,u_a)$ with $|{U}|<\infty$.
\end{assumption}
Then, following the standard optimization result of Theorem 10 \cite{DBLP:journals/jmlr/AgarwalKLM21}, we prove that the single agent policy  converges to the global optimum for the softmax parameterization in Lemma \ref{lemma:ipi}.
\begin{lemma}[Individual Policy Improvement]
Let  the joint action values remain unchanged during the policy improvement.
Let $\boldsymbol{\pi}^0=[\pi_a^0]_{a=1}^n$ be the initial joint policy and $\psi _{s,a}$ be the the corresponding vector of $\psi_{s,u_a}$ for all $u_a\in U$. 	The update for agent $a$ in state $s$ at iteration $t$ with the stepsize $\eta\leq \frac{(1-\gamma)^3}{8}$  is defined  as follows:
\begin{align}
	\psi _{s,a}^{t+1}  =& \psi _{s,a}^{t}+\eta \nabla _{\psi _{s,a}^{t}}V_{s,a}^{t}\left(  \pi _{a}^{t},\boldsymbol{\pi }_{-a}^{0}  \right), \nonumber 
\end{align}
where 
\begin{align}
	&	V_{s,a}^{t}\left(  \pi _{a}^{t},\boldsymbol{\pi }_{-a}^{0}  \right)= \sum_{{u}_a}{\pi _{a}^{t}\left( u_a|s \right)}M_a^{PPG}\left( s,u_a,\boldsymbol{\pi }_{-a}^{0} \right).\nonumber
\end{align}
Then, we have that $ \pi _{a}^{t}\rightarrow  \pi _{a}^{*}$  as $t\rightarrow \infty$, where $\arg\max_{u_a}\pi_a^{*}(u_a|s)=\arg\max_{u_a}M_a^{PPG}\left( s,u_a,\boldsymbol{\pi }_{-a}^{0} \right)$.
\label{lemma:ipi}
\end{lemma}
\begin{proof}
See Appendix B.
\end{proof}
\begin{figure*}[htp]
\centering
\includegraphics[width=\textwidth]{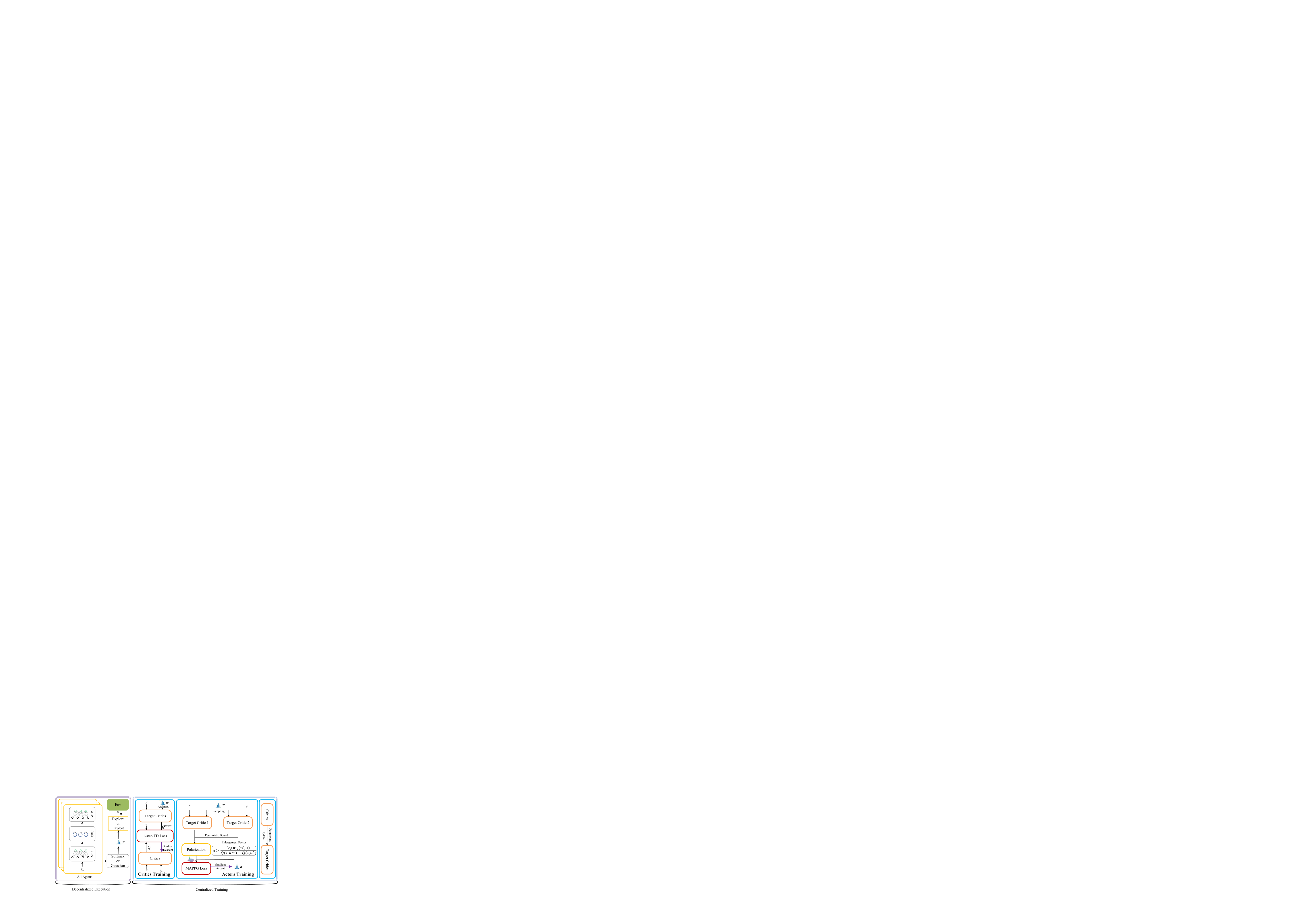}
\caption{MAPPG framework.}
\label{fig:MAPPG}
\end{figure*}
Finally, we prove that optimal individual  policies can be attained as long as MAPPG applies individual policy improvement  to all agents in Theorem \ref{thm:jpi}.
\begin{theorem}[Joint Policy Improvement]
Let  the joint action values remain unchanged during the policy improvement. Let $\boldsymbol{u}^{*}=\arg\max_{\boldsymbol{u} \in \boldsymbol{U} }Q(s,\boldsymbol{u})$ and $\boldsymbol{u}^{sec}=\arg\max_{\boldsymbol{u} \in (\boldsymbol{U}-\{\boldsymbol{u}^{*}\})  }Q(s,\boldsymbol{u})$. Let  the joint  policy at iteration t be $\boldsymbol{\pi}^t=[\pi_a^t]_{a=1}^n$. If  individual policy improvement is applied to each agent $a\in A$ and $\alpha >\max_{a\in A} \frac{\log \boldsymbol{\pi }^{0}_{-a}\left( \boldsymbol{u}_{-a}^{*}|s \right)}{Q\left( s,\boldsymbol{u}^{\sec} \right) -Q\left( s,\boldsymbol{u}^* \right)}$, then we have  that $ \boldsymbol{\pi}^{t}\rightarrow  \boldsymbol{\pi}^{*}$  as $t\rightarrow \infty$, where $\arg\max_{\boldsymbol{u}}\boldsymbol{\pi}^{*}(\boldsymbol{u}|s)=\arg\max_{\boldsymbol{u}}Q\left( s,\boldsymbol{u} \right)$ .
\label{thm:jpi}
\end{theorem}
\begin{proof}
See Appendix B.
\end{proof}
Although Theorem \ref{thm:jpi} requires that when optimizing the policy of a single agent,  other agents' policies should be maintained as $\boldsymbol{\pi}_{-a}^0$. MAPPG replaces $\boldsymbol{\pi}_{-a}^0$ with the current policies $\boldsymbol{\pi}_{-a}^t$ of other agents  in practice, which is a more efficient sampling strategy. This change does not  compromise optimality empirically.

\subsection{MAPPG Architecture}
\label{sec:mappg_architecture}
The overall framework of MAPPG is illustrated in Figure \ref{fig:MAPPG}. For each agent $a$, there is an individual actor $\pi_a(u_a|\tau_a)$ parameterized by $\theta_a$. We denote the joint policy as $\boldsymbol{\pi}=[\pi_a]_{a=1}^n$. Two centralized components are critics $Q(s,\boldsymbol{u})$ and target critics  $Q^{target}(s,\boldsymbol{u})$, parameterized by $\phi$ and $\phi^{-}$. 

At the execution phase, each agent selects actions w.r.t. the current policy and exploration based on the local observation  in a decentralized manner. By interacting with the environment, the transition tuple  $e=\left( s,[z_a]_{a=1}^n,\boldsymbol{u},r,s^{\prime},[z_a^{\prime}]_{a=1}^n \right)$ is added to the buffer.

At the training phase, mini-batches of experiences are sampled from the buffer uniformly at random. We train the parameters of critics to minimise  the 1-step TD loss   by descending their gradients according to:
\begin{align}
\nabla_{\phi} L_{td}\left( \phi \right) =	\nabla_{\phi} \mathbb{E}_{e \sim  D }\left[ \left( y-Q\left( s,\boldsymbol{u} \right) \right) ^2 \right] ,
\label{equation:tdloss}
\end{align}
where $y=r+\gamma Q^{target}\left( s^{\prime},{\arg\max}_{\boldsymbol{u}^{\prime}}\ \boldsymbol{\pi }\left( \boldsymbol{u}^{\prime}|\tau ^{\prime} \right) \right) $.
In the proof of Theorem \ref{thm:jpi}, some strong constraints need to be satisfied. To derive a practical algorithm, we must make approximations. First,  MAPPG adopts target critics for individual policy improvement as an implementation of fixed joint action values. Second, the estimation of the Q-value  has aleatoric  uncertainty and epistemic  uncertainty. Without a constraint, maximization of $J(\boldsymbol{\pi})$ with polarization would lead to an excessively large policy update; hence, we now consider how to modify the objective. We apply the pessimistic bound of $Q^{PPG}$ with the help of two target critics $\{Q_1^{target},\ Q_2^{target}\}$ and penalize changes to the policy that make $Q^{PPG}$ larger than $L$. We train two critics, which are learned with the same training setting except for the initialization parameters. The two target critics share the same network structure as that of the two critics, and the parameters of the two target critics are periodically synchronized with those of the two critics, respectively.  We train the parameters of actors to maximize  the expected polarization Q-function which is called the MAPPG loss  by ascending their gradients according to:
\begin{align}
\nabla _{\theta}J\left( \boldsymbol{\pi } \right) =& \frac{1}{\beta}\mathbb{E}_{\boldsymbol{\pi }}\left[ \sum_a{\nabla _{\theta_a}\log \pi_{a}\left( u_a\mid \tau _a \right)}\right.\nonumber\\
& \left.\min \left( \hat{Q}^{PPG}\left( s,\boldsymbol{u} \right) ,\,\,L \right) \right] ,
\label{eq:pg_loss}
\end{align}
where 
\begin{align}
\hat{Q}^{PPG}(s,\boldsymbol{u})=& \exp\left(\alpha \left( \min_{k\in \left\{ 1,2 \right\}}Q_{k}^{target}\left( s,\boldsymbol{u} \right) \right.\right. \nonumber \\ 
&  \left.\left. -\max_{k\in \left\{ 1,2 \right\}}Q_{k}^{target}\left( s,\boldsymbol{u}_{curr} \right) \right)\right) .\nonumber
\end{align}
To prevent vanishing gradients  caused by the increasing action probability with an inappropriate learning rate in practice, the gradients for the joint action $\boldsymbol{u}$ are set to 0 if $\hat{Q}^{PPG}(s,\boldsymbol{u})<1$ or $\forall a, \pi_a(u_a|\tau_a)>P$ where $P\geq0.5$,  which is called the policy gradient clipping.
For completeness, we summarize the training of MAPPG in the Algorithm \ref{alg:MAPPG}. More details are included in Appendix C.
\begin{algorithm}[!h]
\caption{MAPPG}
\label{alg:4}
\begin{algorithmic}[1]
	\FOR{${episode} = 1$ to ${max\_training\_episode}$}
	\STATE Initialize the environment
	\FOR{$t = 1$ to $max\_episode\_length$}
	\STATE For all agents, get the current state $s$ and observations $[z_a]_{a=1}^n$, choose a joint action $\boldsymbol{u}$ w.r.t. the current policy  and exploration
	\STATE Execute the joint action $\boldsymbol{u}$, observe a reward $r$, and   get the next state $s^{\prime}$ and observations $[z_a^{\prime}]_{a=1}^n$
	\STATE Add the transition $\left( s,[z_a]_{a=1}^n,\boldsymbol{u},r,s^{\prime},[z_a^{\prime}]_{a=1}^n \right)$ to the buffer $D$
	\IF{$time\_to\_update\_actors\_and\_critics$}
	\STATE Sample a random minibatch of $K$ samples from $D$ 
	\STATE Update $\phi$ by descending their gradients according to Equation  (\ref{equation:tdloss})
	\STATE Update $\theta$ by ascending their gradients according to Equation (\ref{eq:pg_loss}) with the policy gradient clipping
	\ENDIF
	\IF {$time\_to\_update\_target\_critics$}
	\STATE Replace target parameters $\phi_i^{-} \gets \phi_i$  for $i\in \{1,2\}$
	\ENDIF
	\ENDFOR
	\ENDFOR
\end{algorithmic}
\label{alg:MAPPG}
\end{algorithm}
\section{Experiments}
\label{sec:experiments}
In this section, first, we empirically study the optimality of MAPPG for discrete and continuous action spaces.  Then, in StarCraft II,  we demonstrate that MAPPG outperforms state-of-the-art MAPG algorithms.  By ablation studies, we verify the effectiveness of our polarization joint action values and the pessimistic bound of $Q^{PPG}$. We further perform an ablation study to verify  the effect of different  enlargement factors  on convergence to the optimum. 
In  matrix and differential games, the worst result in the five training runs with different random seeds is selected to exclude the influence of random initialization parameters of the neural network. All the learning curves  are plotted based on five training runs with different random seeds using mean and standard deviation in StarCraft II and  the ablation experiment.

\subsection{Matrix Game and Differential Game}
In the discrete matrix and continuous differential games, we investigate whether MAPPG can converge to optimal compared with existing MAPG algorithms, including MADDPG \cite{maddpg}, COMA \cite{coma}, DOP \cite{dop}, FACMAC \cite{facmac}, and FOP \cite{fop}. The two games have one common characteristic: some destructive penalties are around with the optimal solution \cite{fop}, which triggers the issue of the \textit{centralized-decentralized mismatch}.
\begin{table}[!h]\scriptsize 	
\begin{subtable}{0.48\linewidth}
	\begin{tabularx}{1\linewidth}{ 
			| >{\centering\arraybackslash}X 
			|| >{\centering\arraybackslash}X 
			| >{\centering\arraybackslash}X | >{\centering\arraybackslash}X| }
		\hline
		\diagbox[width=1.8\linewidth]{$u_1$}{$u_2$}  & A & B & C \\
		\hline
		\hline
		A & \textbf{15} & -12 & -12 \\
		\hline
		B & -12 & 10 & 10 \\  
		\hline
		C &-12 & 10 & 10 \\  
		\hline
	\end{tabularx}
	\caption{Payoff of matrix game}
	\label{tb:payoff}
\end{subtable}
\hfill
\begin{subtable}{0.48\linewidth}
	\begin{tabularx}{1\linewidth}{ 
			| >{\centering\arraybackslash}X 
			|| >{\centering\arraybackslash}X 
			| >{\centering\arraybackslash}X | >{\centering\arraybackslash}X| }
		\hline 
		\diagbox[width=1.8\linewidth]{$\pi_1$}{$\pi_2$}  & 0.0(A) & \textbf{0.8}(B) & 0.2(C) \\
		\hline
		\hline
		0.0(A) & 14.9 & -11.8 & -11.8 \\
		\hline
		\textbf{0.7}(B) & -12.1 & \textbf{9.9} & 9.9 \\  
		\hline
		0.3(C) &-12.0 & 9.9 & 9.9 \\
		\hline
	\end{tabularx}
	\caption{COMA: $\pi_1,\pi_2,Q$}
	\label{tb:coma}
\end{subtable}
\hfill
\begin{subtable}{0.48\linewidth}
	\begin{tabularx}{1\linewidth}{ 
			| >{\centering\arraybackslash}X 
			|| >{\centering\arraybackslash}X 
			| >{\centering\arraybackslash}X | >{\centering\arraybackslash}X| }
		\hline
		\diagbox[width=1.8\linewidth]{ $\pi_1$ }{    $\pi_2$ }  & 0.0(A) & \textbf{0.7}(B) & 0.3(C) \\
		\hline
		\hline
		0.0(A) & -32.5 & -11.2 & -11.2 \\
		\hline
		\textbf{0.9}(B) & -11.4 & \textbf{9.9} & 9.9 \\   
		\hline
		0.1(C) &-11.5 & 9.9 & 9.9 \\   
		\hline
	\end{tabularx}
	\caption{DOP: $\pi_1,\pi_2,Q$}
	\label{tb:dop}
\end{subtable}
\quad
\begin{subtable}{0.48\linewidth}
	\begin{tabularx}{1\linewidth}{ 
			| >{\centering\arraybackslash}X 
			|| >{\centering\arraybackslash}X 
			| >{\centering\arraybackslash}X | >{\centering\arraybackslash}X| }
		\hline
		\diagbox[width=1.8\linewidth]{ $\pi_1$ }{ $\pi_2$ }  &  0.0(A) & \textbf{0.5}(B) & 0.5(C) \\
		\hline
		\hline
		0.0(A) & -11.5 & -11.5 & -11.5 \\
		\hline
		\textbf{0.5}(B) & -11.5 & \textbf{9.9} & 9.9 \\  
		\hline
		0.5(C) &-11.5 & 9.9 & 10.0 \\  
		\hline
	\end{tabularx}
	\caption{FACMAC: $\pi_1,\pi_2,Q$}
	\label{tb:facmac}
\end{subtable}
\hfill
\begin{subtable}{0.48\linewidth}
	\begin{tabularx}{1\linewidth}{ 
			| >{\centering\arraybackslash}X 
			|| >{\centering\arraybackslash}X 
			| >{\centering\arraybackslash}X | >{\centering\arraybackslash}X| }
		\hline
		\diagbox[width=1.8\linewidth]{ $\pi_1$ }{ $\pi_2$ }  & 0.0(A) & 0.0(B) &  \textbf{1.0}(C) \\
		\hline
		\hline
		0.0(A) & 9.1 & -4.4 & -7.4 \\
		\hline
		0.0(B) & -5.3 & 9.0 & 5.0 \\  
		\hline
		\textbf{1.0}(C) &-2.6 & 9.0 & \textbf{9.7} \\  
		\hline
	\end{tabularx}
	\caption{FOP: $\pi_1,\pi_2,Q$}
	\label{tb:fop}
\end{subtable}
\hfill
\begin{subtable}{0.48\linewidth}
	\begin{tabularx}{1\linewidth}{ 
			| >{\centering\arraybackslash}X 
			|| >{\centering\arraybackslash}X 
			| >{\centering\arraybackslash}X | >{\centering\arraybackslash}X| }
		\hline
		\diagbox[width=1.8\linewidth]{ $\pi_1$ }{ $\pi_2$ }  & \textbf{0.4}(A)& 0.3(B) & 0.3(C) \\
		\hline
		\hline
		\textbf{0.4}(A) & \textbf{15.2} & -12.2 & -12.2 \\
		\hline
		0.3(B) & -12.0 & 10.0 & 10.0 \\  
		\hline
		0.3(C) &-12.0 & 10.0 & 10.0 \\  
		\hline
	\end{tabularx}
	\caption{MAPPG: $\pi_1,\pi_2,Q$}
	\label{tb:mappg}
\end{subtable}
\caption{ The cooperative matrix game. Boldface means the optimal/greedy actions from individual policies.}
\label{tb:matrix}
\end{table}
\subsubsection{Matrix Game}
The matrix game is shown in Table \ref{tb:matrix} (a), which is the modified  matrix game proposed by QTRAN \cite{qtran}. This matrix game captures a  cooperative multi-agent task where we have two agents with three actions each. We show the results of COMA, DOP, FACMAC, FOP, and MAPPG  over 10$k$ steps, as in Table \ref{tb:matrix} (b) to \ref{tb:matrix} (f). 
MAPPG uses an $\epsilon$-greedy policy where $\epsilon$ is annealed from 1 to 0.05 over 10$k$ steps. MAPPG is the only algorithm that can successfully converge to the optimum. The results of these algorithms on the original matrix game proposed by QTRAN \cite{qtran}, and  more experiments and details are included in Appendix C.
\begin{figure}[!htp]
\centering
\begin{subfigure}[b]{0.495\linewidth}
	\centering
	\includegraphics[width=\linewidth]{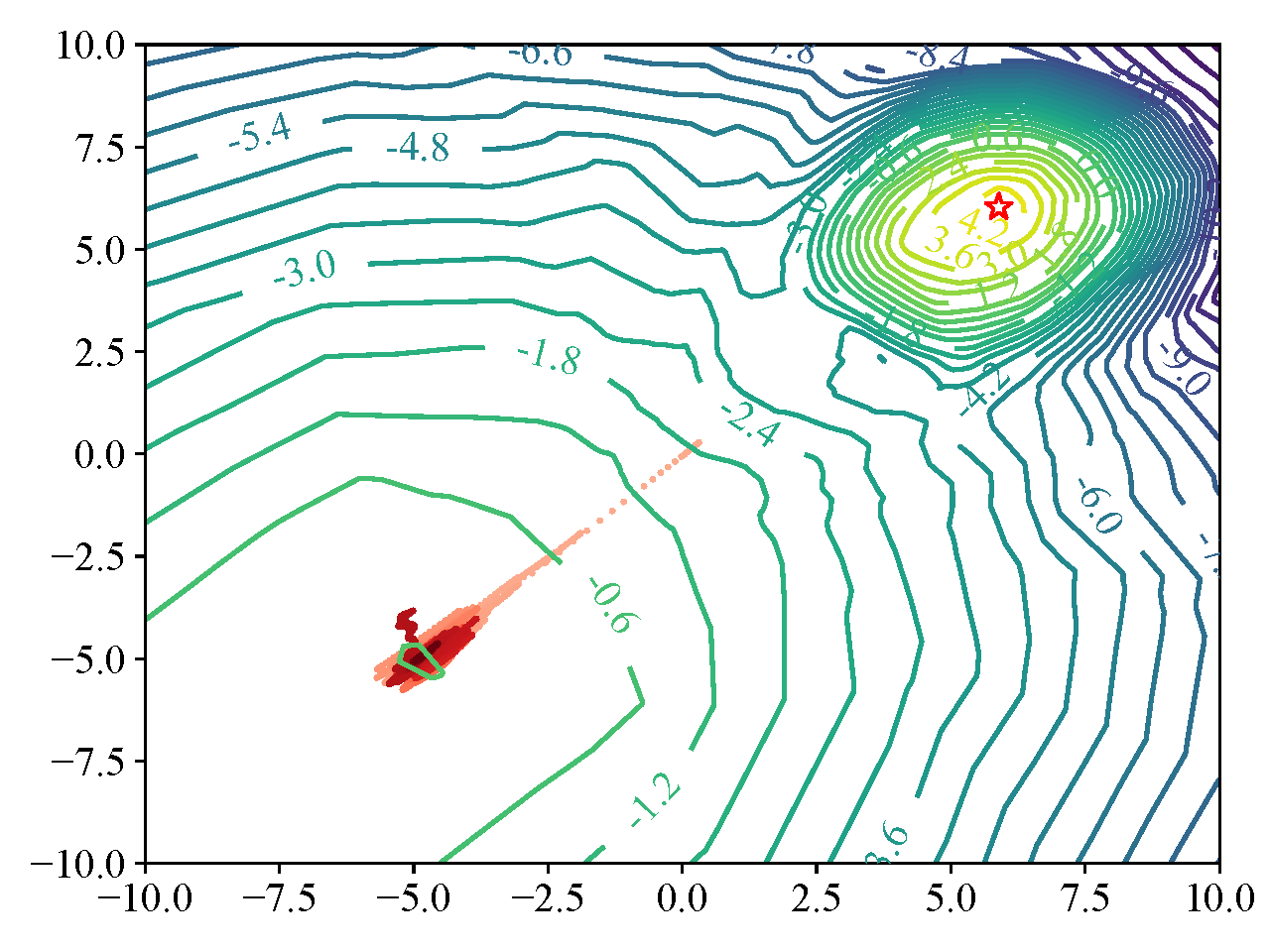}
	\caption{MADDPG}
	\label{fig:maddpg_mtq}
\end{subfigure}
\hfill
\begin{subfigure}[b]{0.495\linewidth}
	\centering
	\includegraphics[width=\linewidth]{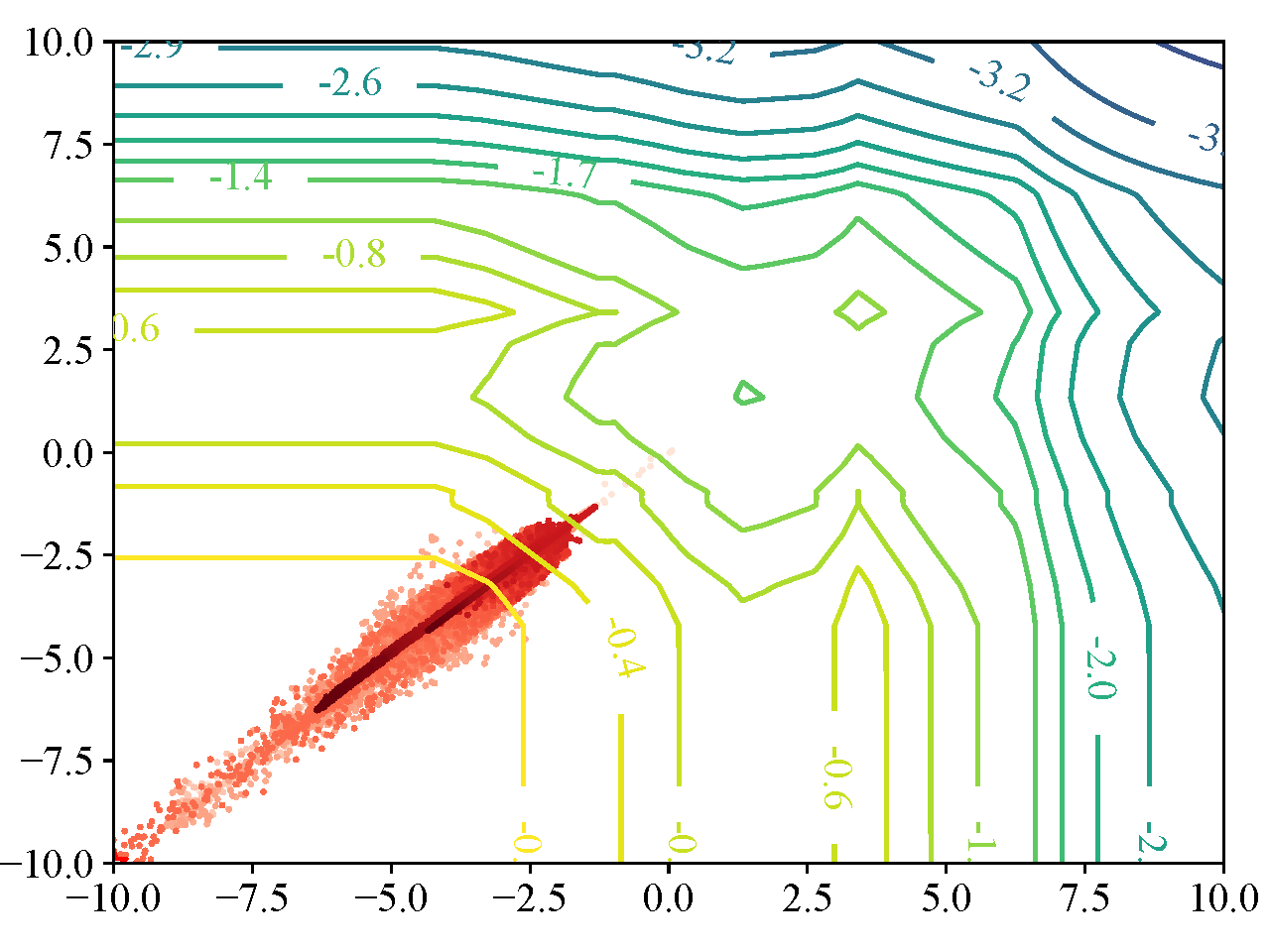}
	\caption{FACMAC}
	\label{fig:facmac_mtq}
\end{subfigure}
\hfill
\begin{subfigure}[b]{0.495\linewidth}
	\centering
	\includegraphics[width=\linewidth]{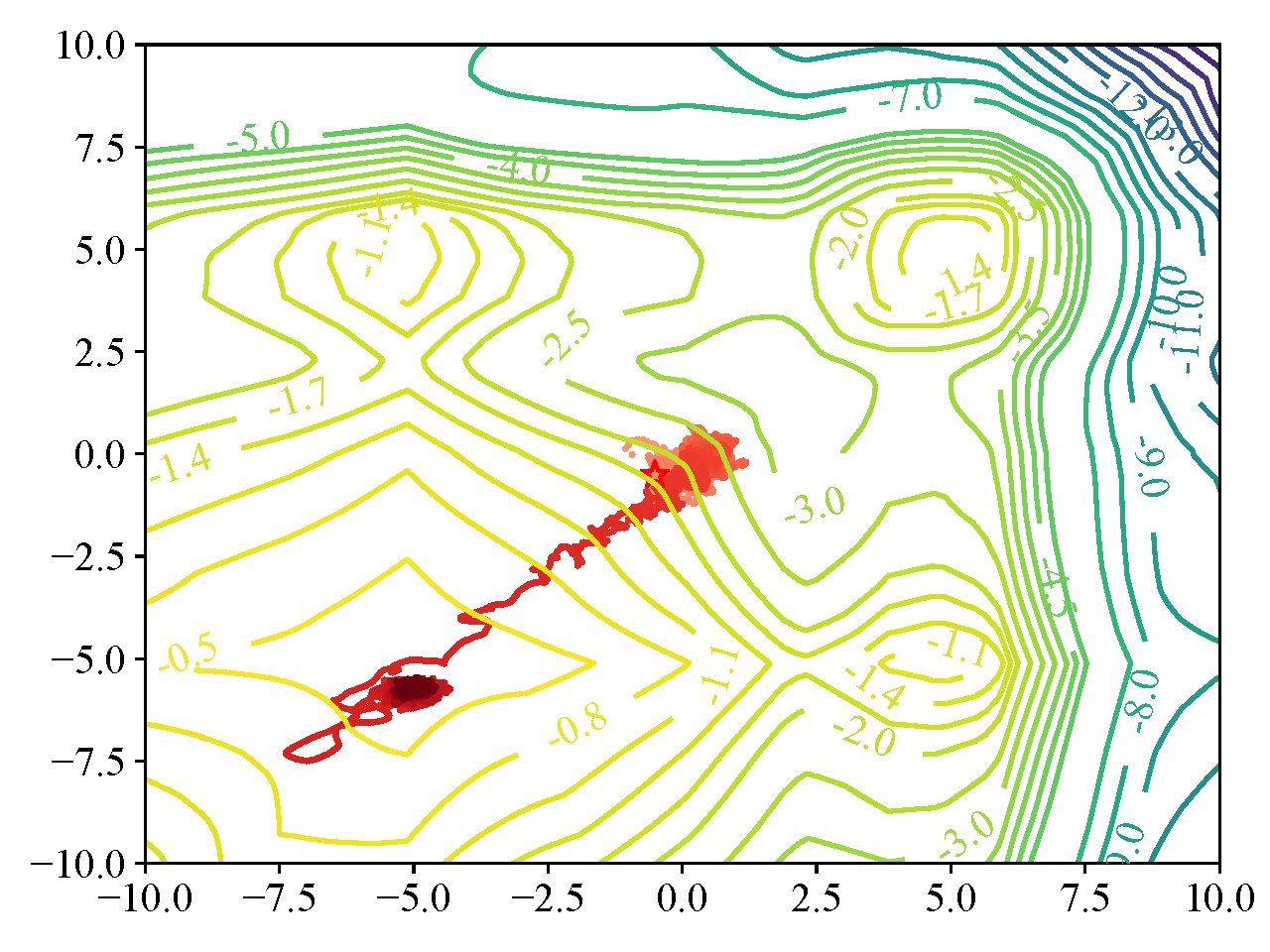}
	\caption{FOP}
	\label{fig:fop_mtq}
\end{subfigure}
\hfill
\begin{subfigure}[b]{0.495\linewidth}
	\centering
	\includegraphics[width=\linewidth]{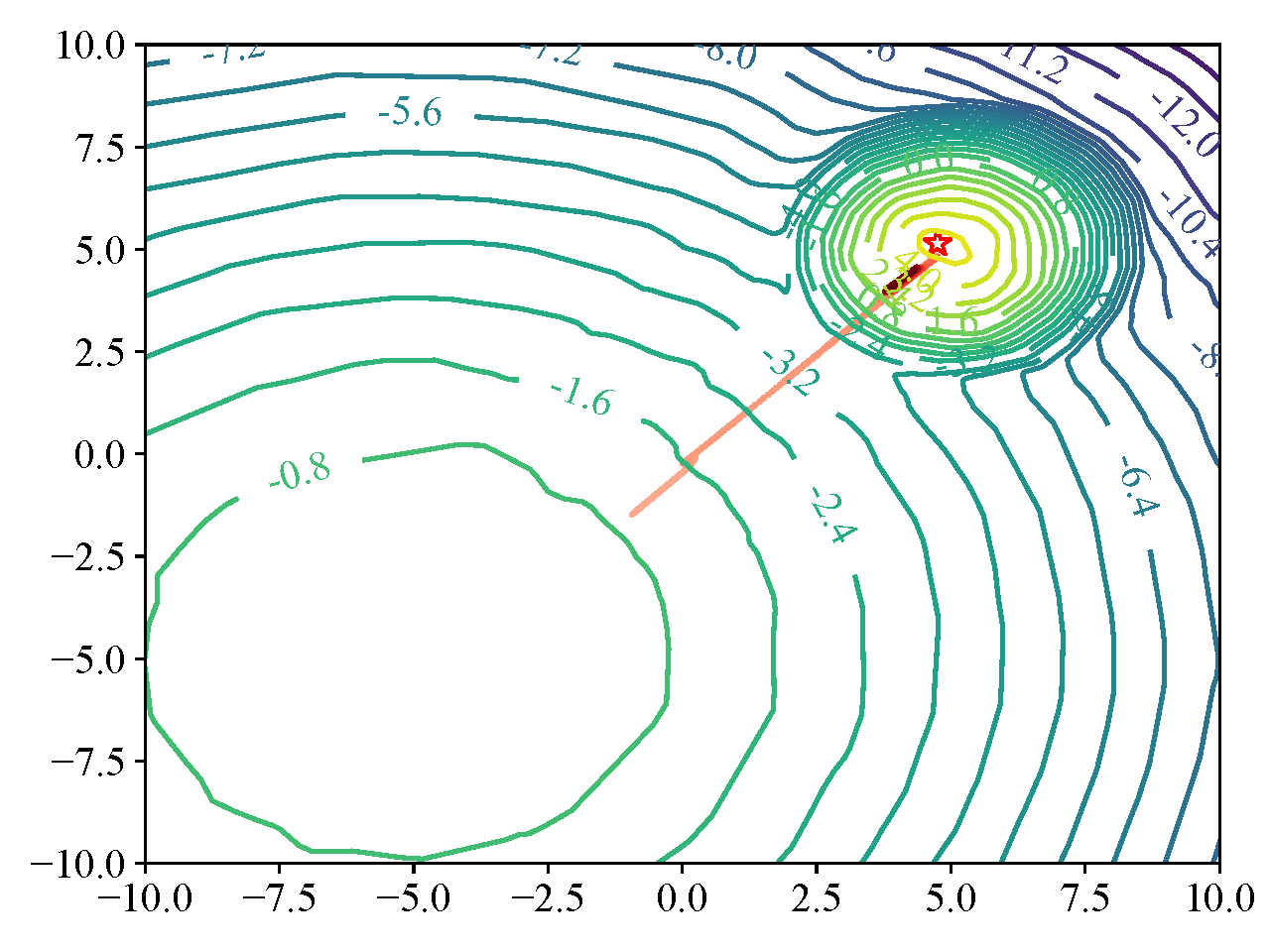}
	\caption{MAPPG}
	\label{fig:mappg_mtq}
\end{subfigure}
\caption{Learning paths  of different algorithms in the MTQ game. The horizontal and vertical axes represent $u_1$ and $u_2$, respectively.  All the points on a given contour line are all at the same joint action values. The shallow red and dark red indicate the start and end of the learning paths, respectively. }
\label{fig:mtq}
\end{figure}
\subsubsection{Differential  Game}
The differential game is the modification of  the Max of Two Quadratic (MTQ) Game from previous literature \cite{fop}. This is a  single-state continuous game for two agents, and each agent has a one-dimensional bounded continuous action space ([-10, 10]) with a shared reward function. In Equation (\ref{eq:df}), $u_1$ and $u_2$  are the actions of two agents and $r(u_1, u_2)$ is the shared reward function received by two agents. There is a sub-optimal solution 0 at (-5, -5) and a global optimal solution 5 at (5, 5).
\begin{align}\label{eq:df}
\left\{ \begin{array}{l}
	f_1=0.8\times \left[ -\left( \frac{u_1+5}{5} \right) ^2-\left( \frac{u_2+5}{5} \right) ^2 \right]\\
	f_2=1\times \left[ -\left( \frac{u_1-5}{1} \right) ^2-\left( \frac{u_2-5}{1} \right) ^2 \right] +5\\
	r\left( u_1,u_2 \right) =\max \left( f_1,f_2 \right)\\
\end{array} \right.   .
\end{align}
\begin{figure*}[htp]
\centering
\begin{subfigure}[b]{0.245\linewidth}
	\centering
	\includegraphics[width=\linewidth]{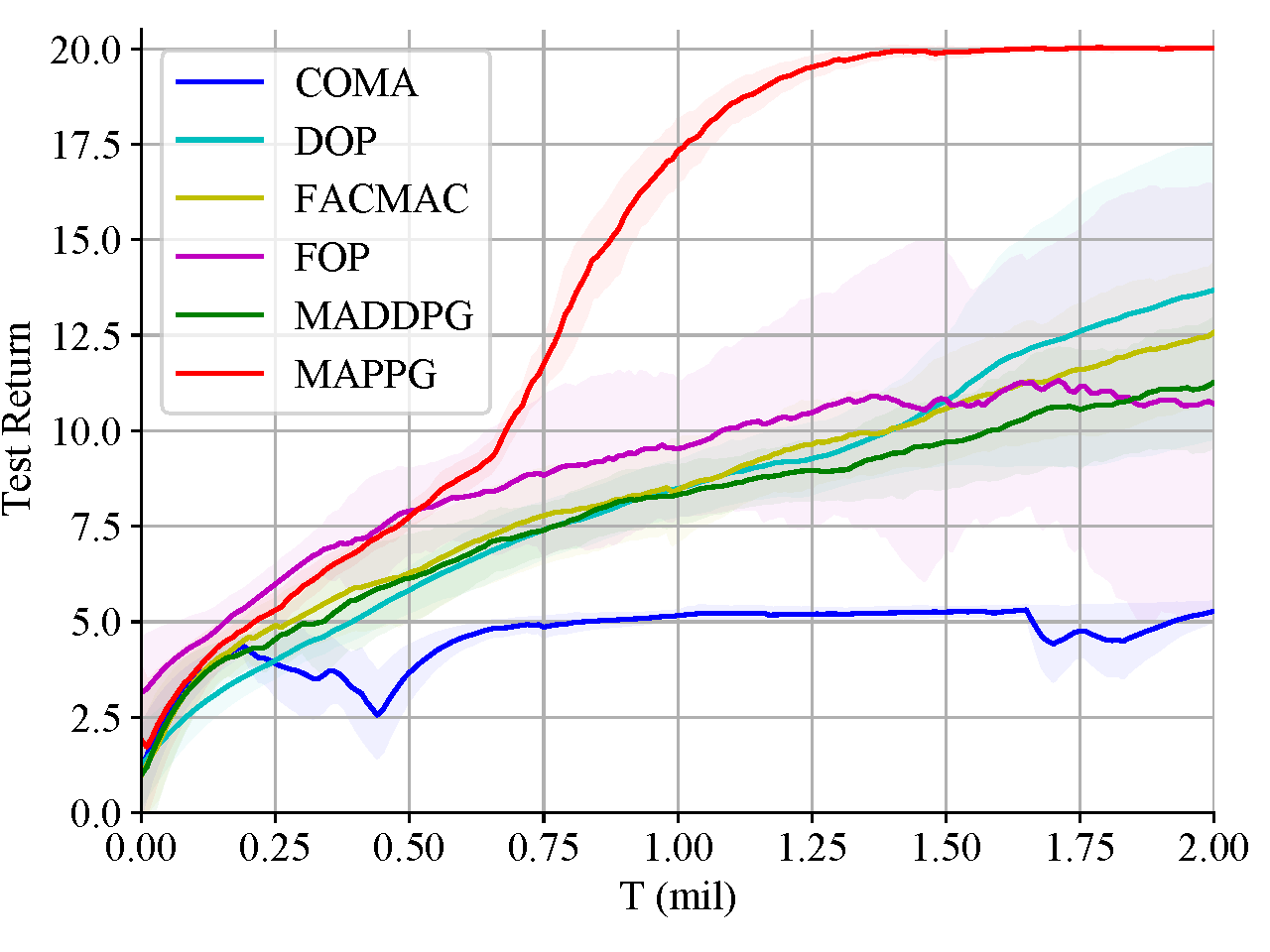}
	\caption{3s\_vs\_4z}
\end{subfigure}
\hfill
\begin{subfigure}[b]{0.245\linewidth}
	\centering
	\includegraphics[width=\linewidth]{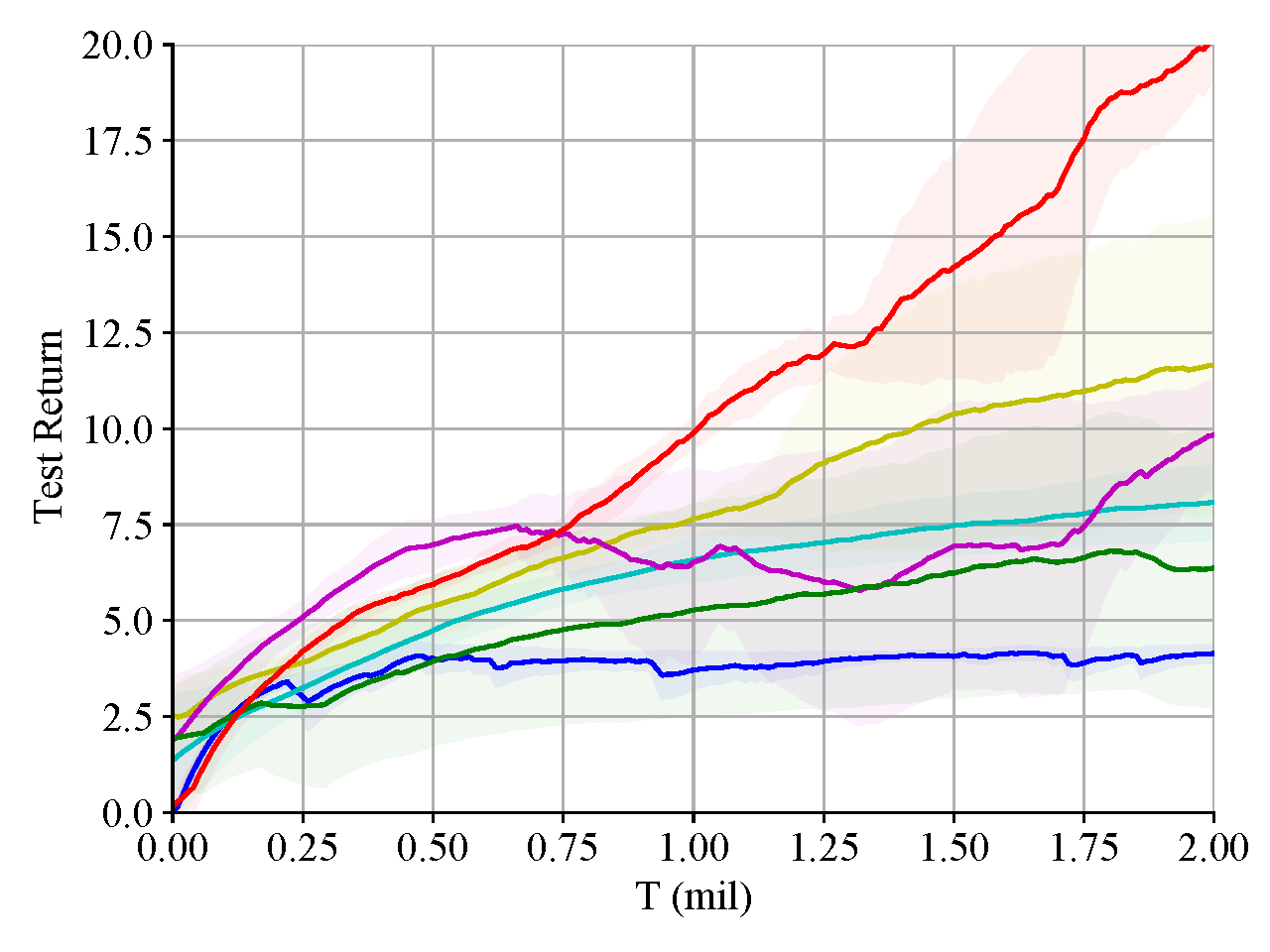}
	\caption{3s\_vs\_5z}
\end{subfigure}
\hfill
\begin{subfigure}[b]{0.245\linewidth}
	\centering
	\includegraphics[width=\linewidth]{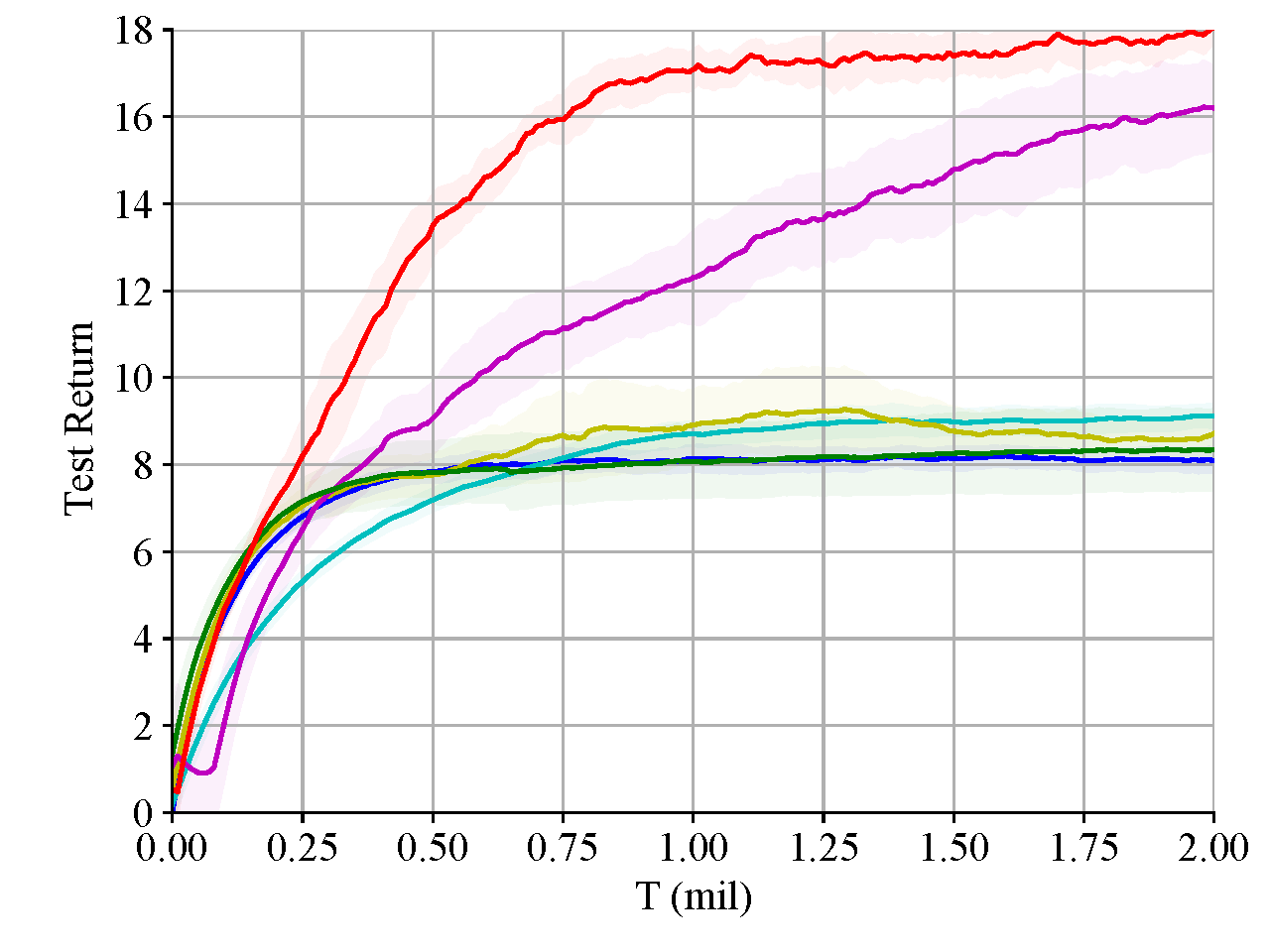}
	\caption{5m\_vs\_6m}
\end{subfigure}
\hfill
\begin{subfigure}[b]{0.245\linewidth}
	\centering
	\includegraphics[width=\linewidth]{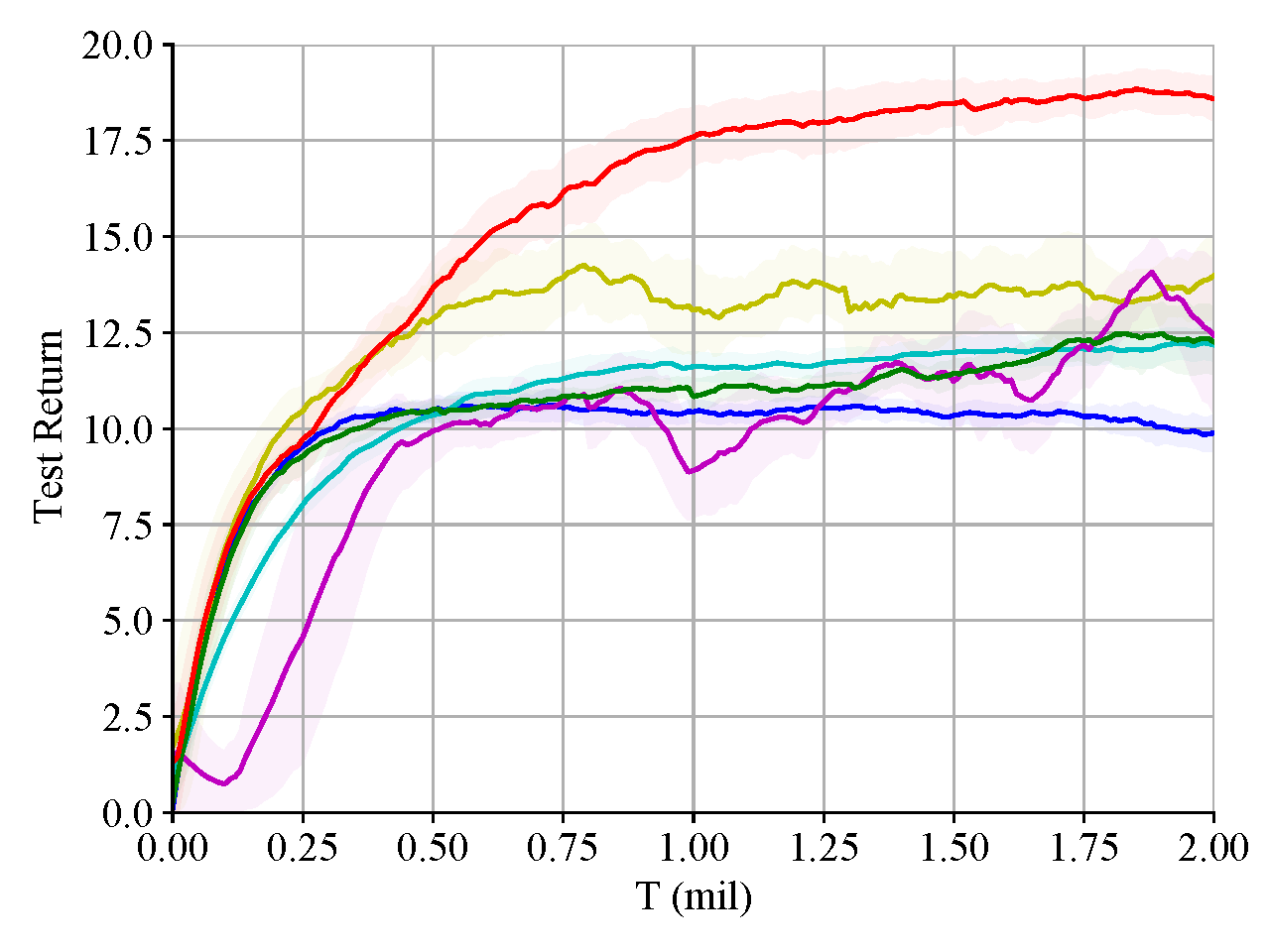}
	\caption{8m\_vs\_9m}
\end{subfigure}
\hfill
\begin{subfigure}[b]{0.245\linewidth}
	\centering
	\includegraphics[width=\linewidth]{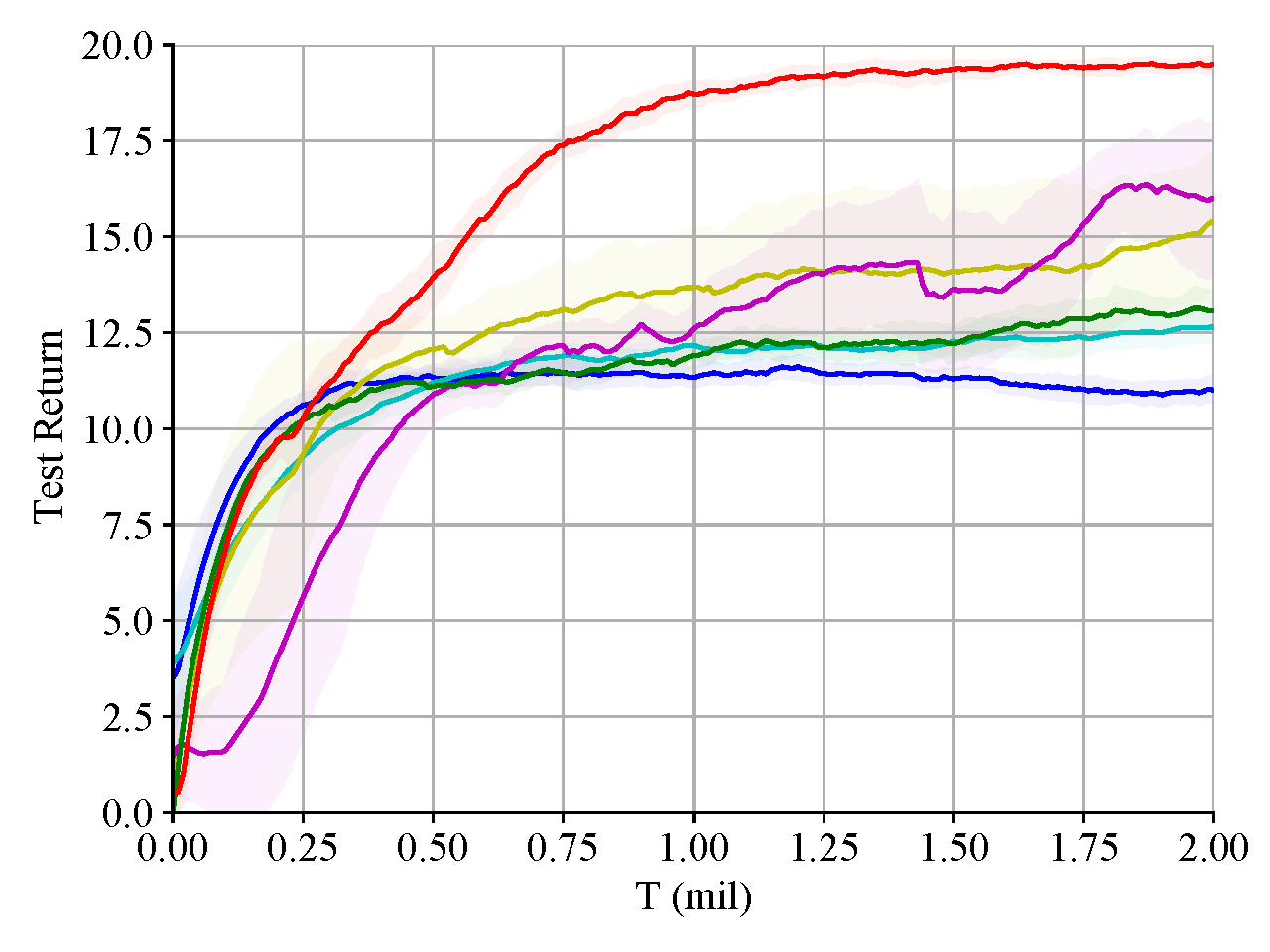}
	\caption{10m\_vs\_11m}
\end{subfigure}
\hfill
\begin{subfigure}[b]{0.245\linewidth}
	\centering
	\includegraphics[width=\linewidth]{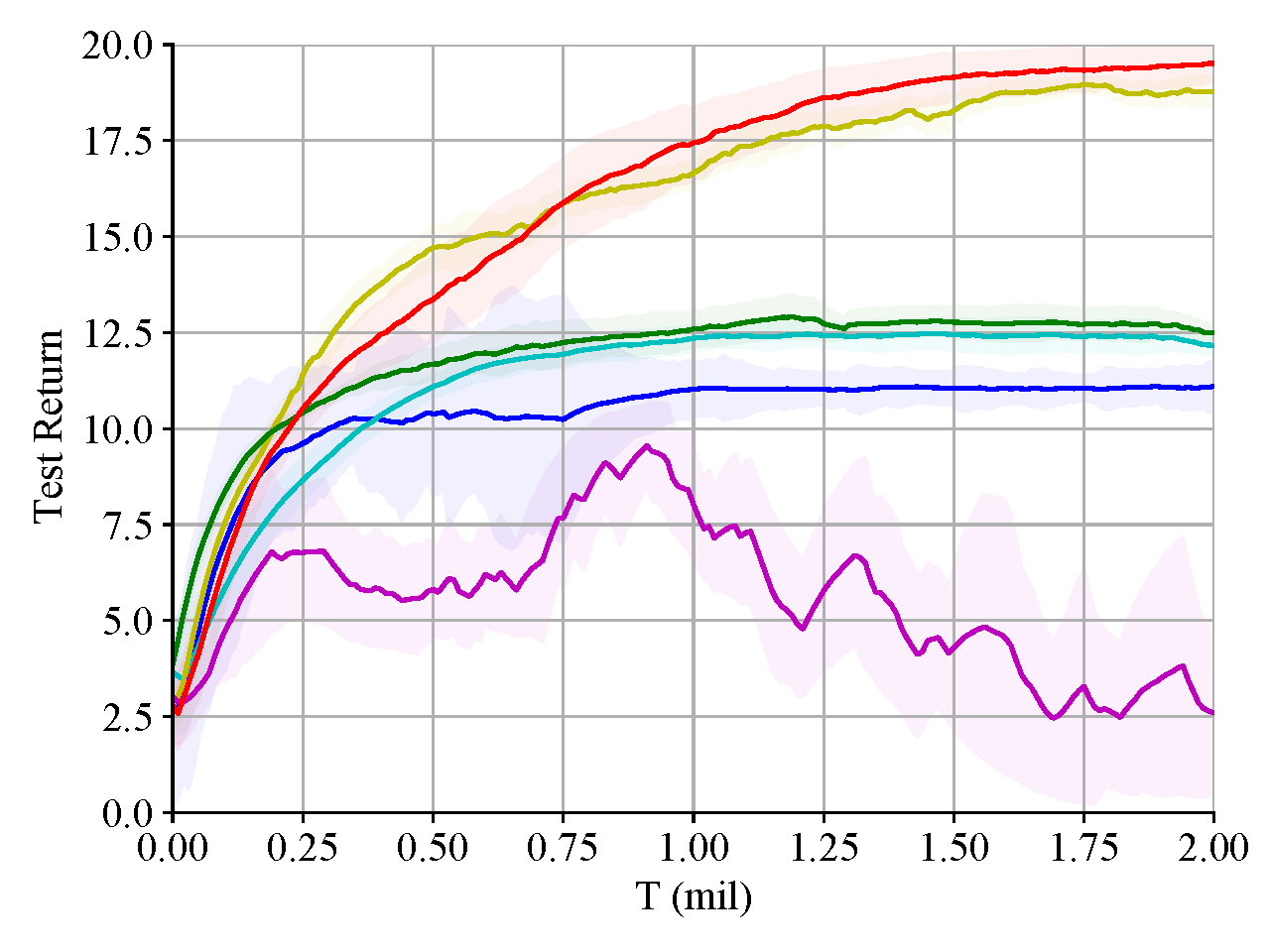}
	\caption{27m\_vs\_30m}
\end{subfigure}
\hfill
\begin{subfigure}[b]{0.245\linewidth}
	\centering
	\includegraphics[width=\linewidth]{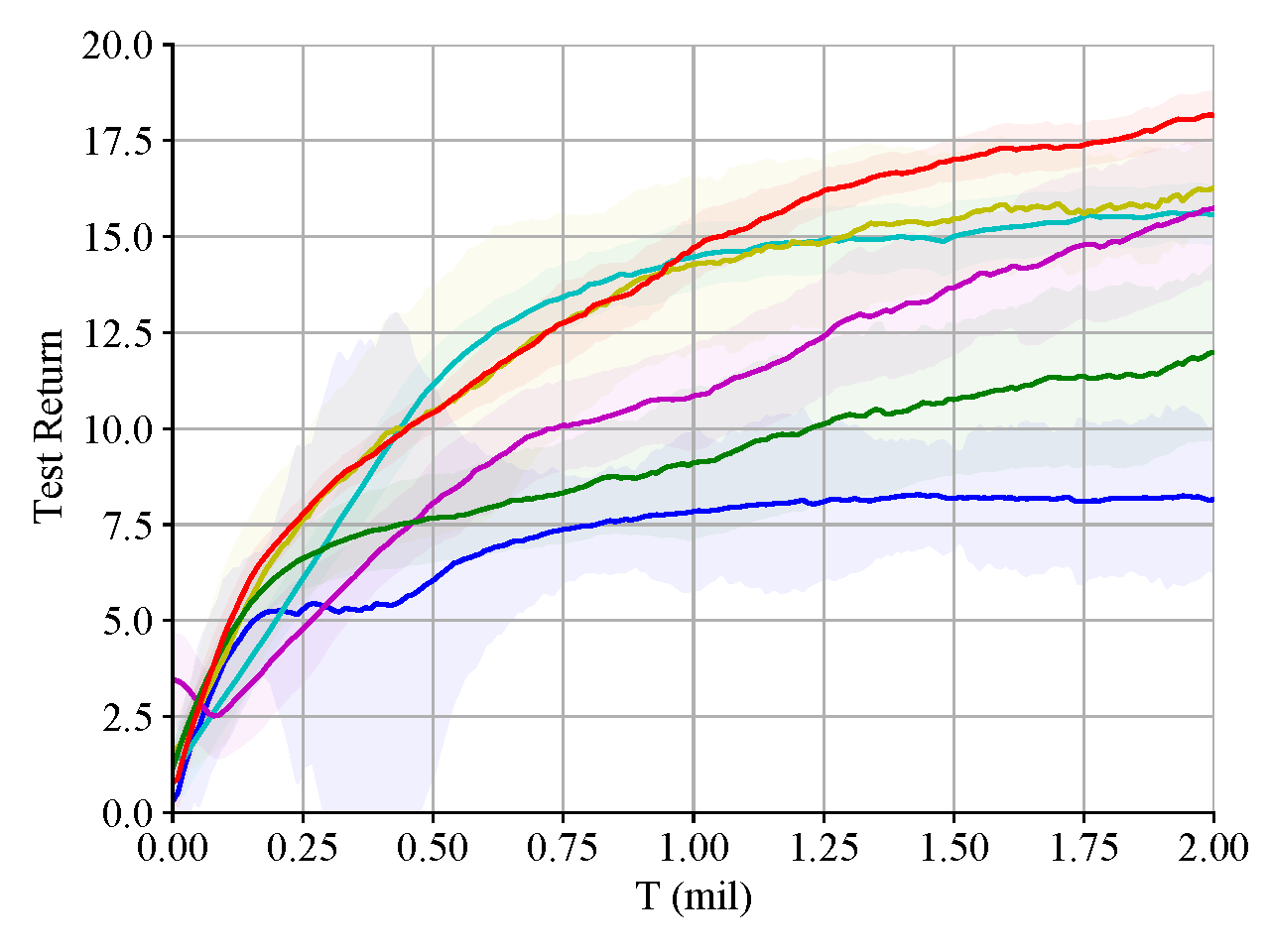}
	\caption{MMM2}
\end{subfigure}
\hfill
\begin{subfigure}[b]{0.245\linewidth}
	\centering
	\includegraphics[width=\linewidth]{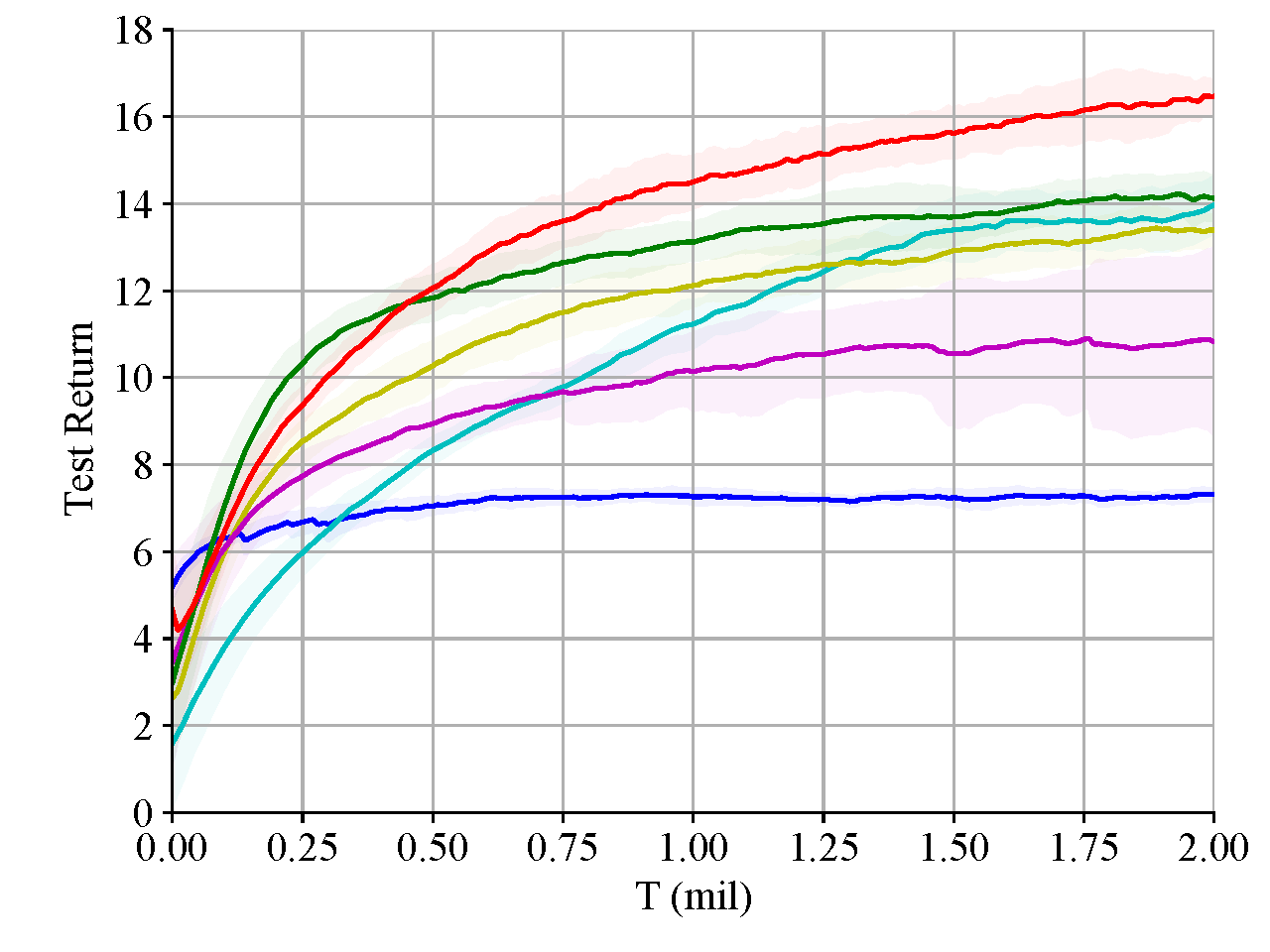}
	\caption{3s5z\_vs\_3s6z}
\end{subfigure}
\caption{ Learning curves of all algorithms in eight maps of StarCraft II.}
\label{fig:smac}
\end{figure*}
In MTQ, we compare MAPPG against existing multi-agent policy gradient algorithms, i.e., MADDPG,  FACMAC, and FOP. Gaussian is used as the action distribution by MAPPG as in SAC \cite{sac}, and the mean of the Gaussian distribution is plotted.  MAPPG uses a similar exploration strategy as in FACMAC, where  actions  are sampled from a uniform random distribution over valid actions up to 10$k$ steps and then from the learned action distribution with Gaussian noise. We use Gaussian noise with mean 0 and standard deviation 1.  The learning paths (20$k$ steps) of all  algorithms are shown as red dots in Figure \ref{fig:mtq} (a) to \ref{fig:mtq} (d).   MAPPG consistently converges to the global optimum while all other baselines fall into the sub-optimum.   MADDPG   can estimate $r(u_1,u_2)$ accurately, but fail to converge to the global optimum.  However, the regular  decomposed actor-critic algorithms (FACMAC and FOP) converge to the sub-optimum and also have limitations to express $r(u_1,u_2)$. The results of these algorithms for the original MTQ game  used by  previous literature \cite{fop}, and  more details are included in Appendix C.
\subsection{StarCraft II}	
We evaluate MAPPG on the challenging StarCraft Multi-Agent Challenge (SMAC) benchmark \cite{smac} in eight maps, including 3s\_vs\_4z, 3s\_vs\_5z, 5m\_vs\_6m, 8m\_vs\_9m, 10m\_vs\_11m, 27m\_vs\_30m, MMM2 and 3s5z\_vs\_3s6z. The baselines include 5 state-of-the-art  MAPG algorithms (COMA, MADDPG, stochastic DOP, FOP, FACMAC). MAPPG uses an $\epsilon$-greedy policy in which $\epsilon$ is annealed from 1 to 0.05 over 50$k$ steps. Results are shown in Figure \ref{fig:smac} and MAPPG outperforms all the baselines in the final performance, which indicates that MAPPG can jump out of sub-optima. More details of the StarCraft II experiments are included in Appendix C.
\subsection{Ablation Studies}
\begin{figure}[htp]	
\centering
\begin{subfigure}[b]{0.495\linewidth}
	\centering
	\includegraphics[width=\linewidth]{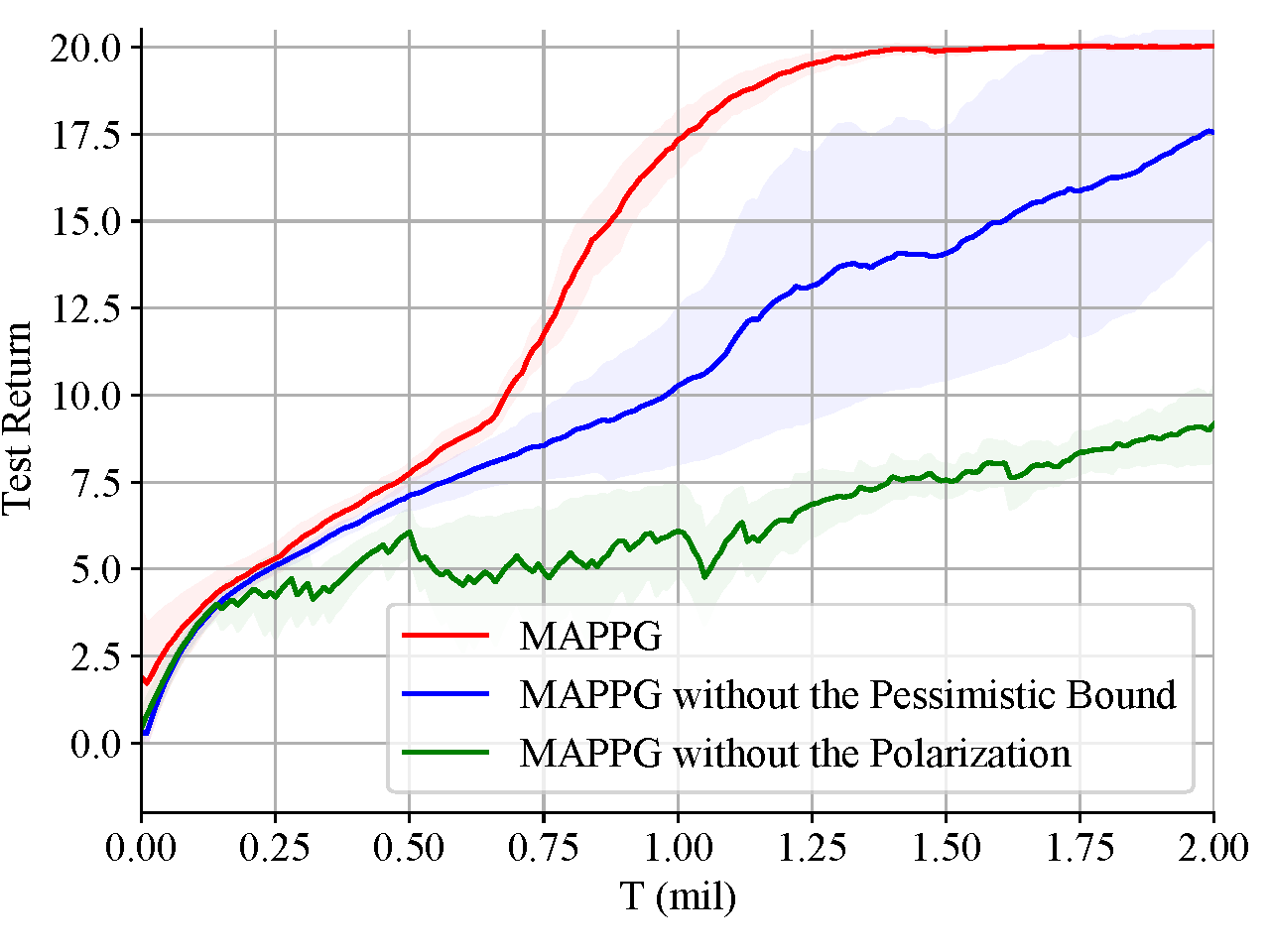}
	\caption{3s\_vs\_4z}
	\label{fig:ablation_smac}
\end{subfigure}
\hfill
\begin{subfigure}[b]{0.495\linewidth}
	\centering
	\includegraphics[width=\linewidth]{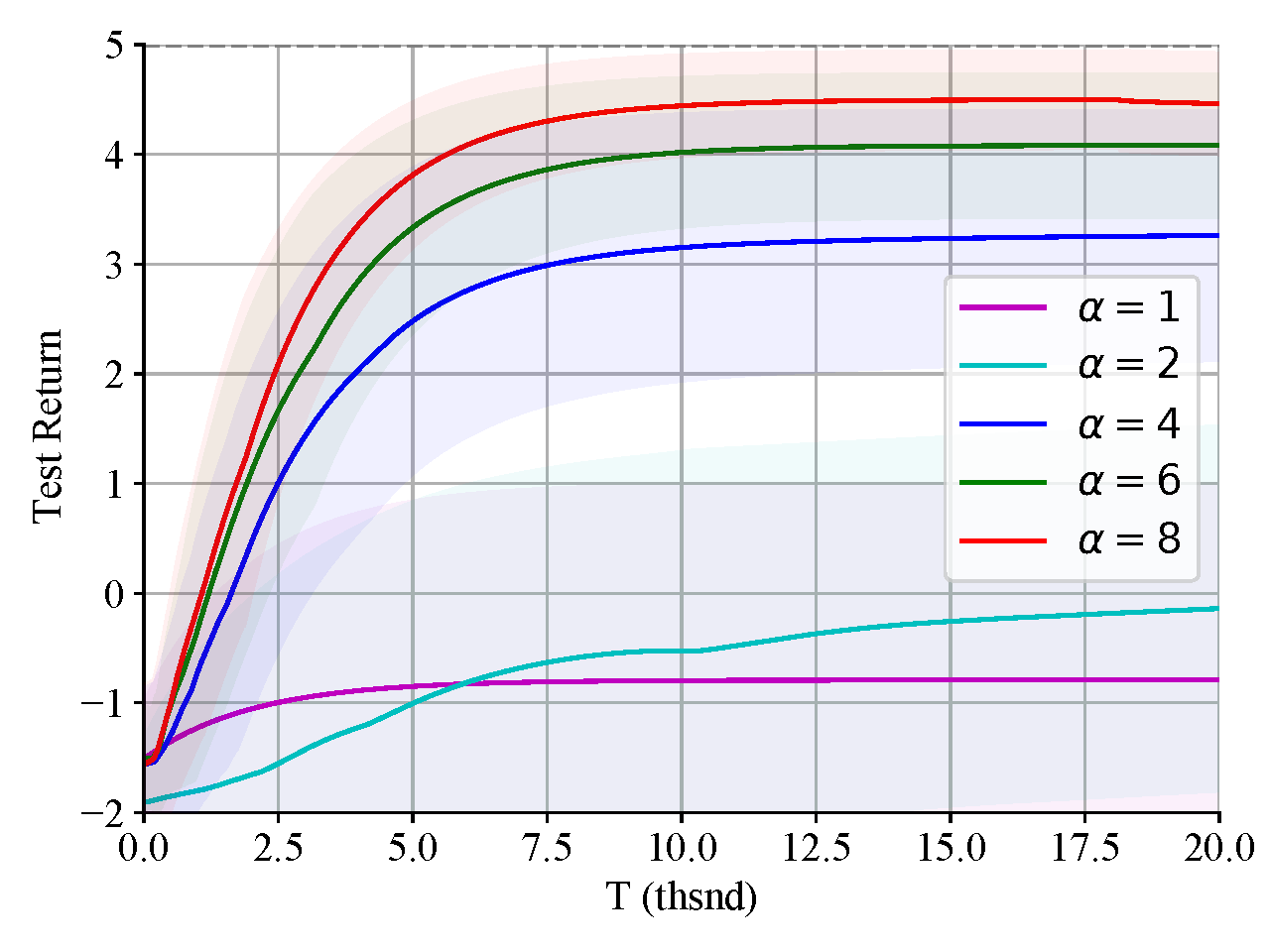}
	\caption{MTQ}
	\label{fig:ablation_mtq}
\end{subfigure}
\caption{ Ablations on the SMAC benchmark and MTQ game.}
\label{fig:ablation}
\end{figure}
In Figure \ref{fig:ablation} (a),  the comparison between the  MAPPG and \textit{MAPPG without the pessimistic bound}  demonstrates the importance of making conservative policy improvements, which can alleviate the problem of excessively large policy updates caused by inaccurate value function estimation. The comparison between the  MAPPG and \textit{MAPPG without the polarization}  demonstrates the contribution of polarization joint  action values to global convergence.  In Figure \ref{fig:ablation} (b),  we  observe that MAPPG can converge to a policy that can obtain a reward closer to the optimal reward as the increase of the  enlargement factor. The discrepancy between the learning curve of $\alpha=\{1,\ 2\}$ and other learning curves indicates  the influence of polarization. The difference in performance of different $\alpha$ is consistent with Theorem \ref{thm:optimality_consistency} and \ref{thm:jpi} even in the continuous action space.
\section{Conclusion}
\label{sec:conclusion}
This paper presents MAPPG, a novel multi-agent actor-critic framework that allows centralized end-to-end training and efficiently learns to do credit assignment properly to enable decentralized execution.  MAPPG takes advantage of the polarization joint action value that efficiently guarantees the  consistency  between  individual optimal actions and the joint optimal  action. Our theoretical analysis shows that MAPPG can converge to optimality. Empirically,  MAPPG  achieves competitive results compared with state-of-the-art MAPG baselines for large-scale complex multi-agent cooperations.

\section*{Acknowledgments}
This work is supported in part by  Science and Technology Innovation 2030 – “New Generation Artificial Intelligence” Major Project (2018AAA0100905), National Natural Science Foundation of China (62192783, 62106100, 62206133), Primary Research \& Developement Plan of Jiangsu Province (BE2021028), Jiangsu Natural Science Foundation (BK20221441), Shenzhen Fundamental Research Program (2021Szvup056), CAAI-Huawei MindSpore Open Fund,  State Key Laboratory of Novel Software Technology Project (KFKT2022B12), Jiangsu Provincial Double-Innovation Doctor Program (JSSCBS20210021, JSSCBS20210539), and in part by the Collaborative Innovation Center of Novel Software Technology and Industrialization. The authors would like to thank the anonymous reviewers for their helpful advice and support.

\bibliography{aaai23}

\end{document}